\documentclass[11pt, a4paper]{scrartcl}

\usepackage[utf8]{inputenc}
\usepackage[T1]{fontenc}

\usepackage[english]{babel}
\usepackage{lmodern}
\usepackage{microtype}
\usepackage{breakcites}
\usepackage[l2tabu, orthodox]{nag}
\usepackage[automark,headsepline]{scrlayer-scrpage}

\usepackage{authblk}
\usepackage[margin=2.5cm]{geometry}
\usepackage{xcolor}

\usepackage[ruled, vlined, linesnumbered]{algorithm2e}
\usepackage{amsmath,amsfonts,amssymb,amsthm,mathtools}

\usepackage[colorlinks, pagebackref, citecolor=black!30!brown, linkcolor=black!30!brown]{hyperref}
\usepackage[nameinlink]{cleveref}

\newcommand{\N}{\mathbb{N}}
\newcommand{\R}{\mathbb{R}}

\newcommand{\cE}{\mathcal{E}}
\newcommand{\cO}{\mathcal{O}}
\newcommand{\cT}{\mathcal{T}}
\newcommand{\cU}{\mathcal{U}}

\newcommand{\teps}{\tilde\epsilon}
\newcommand{\tm}{\mathtt{m}}
\newcommand{\ts}{\mathtt{s}}

\newcommand{\abs}[1]{\left\vert #1 \right\vert}
\newcommand{\norm}[1]{\left\Vert #1 \right\Vert}
\newcommand{\OneTo}[2]{#1\in[#2]}
\newcommand{\ball}{\mathtt{B}}

\DeclareMathOperator{\km}{km}
\newcommand{\wmin}{w_{\min}}
\newcommand{\wmax}{w_{\max}}
\newcommand{\one}{\mathbf{1}}
\newcommand{\TnegX}{\Theta_{(K,\epsilon)}(X)}
\newcommand{\TnegS}{\Theta_{(K,\epsilon)}(S)}
\newcommand{\wfuncN}[2]{#1: #2 \rightarrow\N}

\newtheoremstyle{thmstyle}
  {3pt}{3pt}{\itshape}{}{\bfseries}{}{.5em}
  {\thmname{#1}\thmnumber{ #2}\thmnote{ \textmd{(#3)}}}
\theoremstyle{thmstyle}

\newtheorem{definition}{Definition}
\newtheorem{theorem}[definition]{Theorem}
\newtheorem{lemma}[definition]{Lemma}
\newtheorem{corollary}[definition]{Corollary}

\clearpairofpagestyles
\begin{document}

  \title{On Coreset Constructions for the Fuzzy $K$-Means Problem}

\author{Johannes Blömer}
\author{Sascha Brauer}
\author{Kathrin Bujna}
\affil{Department of Computer Science\\Paderborn University\\Paderborn, Germany}
\date{\today}

\maketitle

\begin{abstract}
  The fuzzy $K$-means problem is a popular generalization of the well-known $K$-means problem to soft clusterings.
  We present the first coresets for fuzzy $K$-means with size linear in the dimension, polynomial in the number of clusters, and poly-logarithmic in the number of points.
  We show that these coresets can be employed in the computation of a $(1+\epsilon)$-approximation for fuzzy $K$-means, improving previously presented results.
  We further show that our coresets can be maintained in an insertion-only streaming setting, where data points arrive one-by-one.
\end{abstract}

  \section{Introduction}\label{sec:introduction}

Clustering is a widely used technique in unsupervised machine learning.
The goal is to divide some set of objects into groups, the so-called clusters, such that objects in the same cluster are more similar to each other than to objects in other clusters.
Nowadays, clustering is ubiquitous in many research areas, such as data mining, image and video analysis, information retrieval, and bioinformatics.

The most common approach are hard clusterings, where the input is partitioned into a given number of clusters, i.e.\ each point belongs to exactly one of the clusters.
However, in some applications it is beneficial to be less decisive and allow points to belong to more than one cluster.
This idea leads to so-called \emph{soft clusterings}.
In the following, we study a popular soft clustering problem, the \emph{fuzzy $K$-means} problem.

The fuzzy $K$-means objective function goes back to work by Dunn and Bezdek et al. \cite{dunn73,bezdek84}.
Today, it has found numerous practical applications, for example in data mining \cite{hirota99}, image segmentation \cite{rezaee}, and biological data analysis \cite{dembele}.
Practical applications generally use the fuzzy $K$-means (FM) algorithm, an iterative relocation scheme similar to Lloyd's algorithm \cite{Lloyd1982} for $K$-means, to tackle the problem.
The FM algorithm converges to a local minimum or a saddle point of the objective function \cite{bezdek84,bezdek87}.
Distinguishing whether the FM algorithm has reached a local minimum or a saddle point is a problem which got some attention on its own \cite{kim88,hoppner03}.
Moreover, it is known that the algorithm converges locally, i.e.\ started sufficiently close to a minimizer, the iteration sequence converges to that particular minimizer \cite{hathaway86}.
However, from a theoretician's point of view this algorithm has the major downside that stationary points of the objective function can be arbitrarily worse than an optimal solution \cite{bbb16}.
Currently, the only paper on algorithms with approximation guarantees for the fuzzy $K$-means problem is \cite{bbb16}, where the authors present a PTAS assuming a constant number of clusters.

Clustering is usually applied when huge amounts of data need to be processed.
This has sparked significant interest in researching clustering in a streaming model, where the data does not fit into memory.
A lot of research has been done on this setting for $K$-means.
In a single pass setting, where we are only allowed to read the data set once, the $K$-means objective function can be approximated up to a constant factor, by choosing $\cO(K\log(K))$ means, instead of $K$ \cite{ailon09}.
This was improved to an algorithm computing exactly $K$ means but still maintaining a constant factor approximation \cite{braverman11, shindler11}.
There, the authors considered a setting where points arrive one-by-one and they can use $\cO(K\log(N))$ memory, where $N$ is the total number of points.

The goal of a coreset is to find a small representation of a large data set, retaining the characteristics of the original data.
Coresets have emerged as a key technique to tackle the streaming model.
The idea of maintaining a coreset of the stream is that if, after having read the whole stream, the maintained coreset is small enough to fit into memory, then standard algorithms can be used to solve the problem almost optimally for the points in the stream.
The first coreset construction for $K$-means is due to Har-Peled and Mazumdar, and is of size $\cO(\log(N))$ \cite{harpeled03}.
They also showed how to maintain a coreset, with size poly-logarithmic in $N$, of a data stream, by combining their notion of a coreset with the merge-and-reduce technique by Bentley and Saxe \cite{bentley80}.
They improved their coreset construction to have size independent of $N$ \cite{harpeled05}.
Feldman and Langberg presented a general framework computing coresets for a large class of hard clustering problems with size independent of $N$ \cite{Feldman11a}.
Later, Feldman et al. presented coresets with size independent of $N$ and $D$ by using a construction based on low-rank approximation \cite{Feldman13}.
Furthermore, they generalize Har-Peled and Mazumdar's application of the merge-and-reduce technique, showing how coresets with certain properties can be maintained in a streaming setting.
The results of our paper are based on Chen's sampling based construction, which yields coresets with size poly-logarithmic in $N$, $K$, and $D$ \cite{Chen09}.
Applying the merge-and-reduce technique, Chen's coresets can also be used to maintain a poly-logarithmic sized coreset of a data stream.

There has been some work on applying the fuzzy $K$-means (FM) algorithm to large data sets.
Hore et al. \cite{hore07} presented a single pass variant of the algorithm, which processes the data chunk-wise.
This idea was refined and extended to a single pass and online kernel FM algorithm \cite{havens12}.
However, these are still variants of the FM algorithm, hence provide no guarantees for the quality of solutions.
So far, no coreset constructions have been presented for the fuzzy $K$-means problem, and the literature is not rich on coreset constructions for soft clustering problems, in general.
There is a construction for the problem of estimating mixtures of semi-spherical Gaussians which yields coresets with size independent of $N$ \cite{Feldman11}.
This result was generalized to a large class of hard and soft clustering problems that are based on $\mu$-similar Bregman divergences \cite{Lucic16}.

\subsection{Our Result}

We prove the existence of small coresets for the fuzzy $K$-means problem.
In \Cref{sec:coresets}, we prove that, by adjusting some parameters of Chen's construction \cite{Chen09}, we obtain a coreset for the fuzzy $K$-means problem with size still poly-logarithmic in $N$.
Our proof technique is a non-trivial combination of the notion of negligible fuzzy clusters \cite{bbb16} and weak coresets \cite{Feldman07}.
This results in a general weak-to-strong lemma (cf. \Cref{lem:weaktostrong}), which states that weak coresets for the fuzzy $K$-means problem fulfilling certain conditions are already strong coresets.
Afterwards, we prove that our adaptation of Chen's algorithm yields a weak coreset satisfying all conditions of the weak-to-strong theorem.
In \Cref{sec:appl} we substantiate the usefulness of our result by presenting two applications of coresets for fuzzy $K$-means.
First, we improve the analysis of a previously presented \cite{bbb16} PTAS for fuzzy $K$-means, removing the dependency on the weights of the data points from the runtime.
Running this algorithm on our coreset instead of the original input improves upon the runtime of previously known $(1+\epsilon)$-approximation schemes.
The improvement lies in the exponential term, which we reduce from $N^{\cO(\mathrm{poly}(K,1/\epsilon))}$ to $\log(N)^{\cO(\mathrm{poly}(K,1/\epsilon))}$, while maintaining non-exponential dependence on $D$.
Second, we argue that an application of the merge-and-reduce technique enables us to maintain a fuzzy $K$-means coreset in a streaming model, where point arrive one-by-one.

  \section{Preliminaries}

Let $X\subset\R^D$ be a set of points in $D$-dimensional space and $\wfuncN{w}{X}$ be an integer weight function on the points.
Using integer weights eases the notation of our exposition.
We later argue how our results generalize to rational weights.
Unweighted data sets are denoted by using the weight function $\one$ mapping every input to $1$.
We call $w(X) = \sum_{x\in X} w(x)$ the total weight of $X$ and denote the maximum and minimum weights by $\wmax(X) = \max_{x\in X} w(x)$ and $\wmin(X) = \min_{x\in X} w(x)$.

\begin{definition}[Fuzzy $K$-means]\label{def:fuzzy}
  Let $m\in\R_{>1}$ and $K\in\N$.
  The \emph{fuzzy $K$-means problem} is to find a set of means $M = \{\mu_k\}_{\OneTo{k}{K}}\subset \R^D$ and a membership function $r: X\times [K] \rightarrow [0,1]$ minimizing
  \begin{align*}
    \phi(X,w,M,r) = &\sum_{x\in X} w(x)\sum_{\OneTo{k}{K}} r(x,k)^m \norm{x-\mu_k}^2  \\
  	\text{subject to } & \\
    \forall x\in X: &\sum_{\OneTo{k}{K}}r(x,k) = 1 \ .
  \end{align*}
\end{definition}

The parameter $m$ is called fuzzifier.
It determines the softness of an optimal clustering and is not subject to optimization, since the cost of any solution can always be decreased by increasing $m$.
In the case $m=1$, the cost can not be decreased by assigning membership of a point to any mean except its closest.
Consequently, optimal solutions of the fuzzy $K$-means problem for $m=1$ coincide with optimal solutions for the $K$-means problem on the same instance.
Hence, in the following we always assume $m$ to be some constant larger than $1$.

Similar to the classical $K$-means problem, it is easy to optimize means or memberships of fuzzy $K$-means, assuming the other part of the solution is fixed \cite{bezdek84}.
This means, given some set of means $M$ we call a respective optimal membership function $r_M^*$ induced by $M$ and set $\phi(X,w,M) := \phi(X,w,M,r_M^*)$.
Analogously, given some membership function $r$ we call a respective optimal set of means $M_r^*$ induced by $r$ and set $\phi(X,w,r) := \phi(X,w,M_r^*,r)$.
Finally, given some optimal solution $M^*,r^*$ we denote $\phi^{opt}(X,w) := \phi(X,w,M^*,r^*)$.

\subsection{Fuzzy Clusters}

Recall, that in a soft-clustering there is no partitioning of the input points.
Instead, we describe the $k^{th}$ cluster of a fuzzy clustering as a vector of the fractions of points assigned to it by the membership function.
We denote the size (or the total weight) of the $k^{th}$ cluster by $r(X,w,k) = \sum_{x\in X}w(x)r(x,k)^m$.
Given a set of means $M$, we denote the cost of the $k^{th}$ cluster by $\phi_k(X,w,M,r) = \sum_{x\in X} w(x) r(x,k)^m \norm{x-\mu_k}^2$.

\subsection{\texorpdfstring{$K$}{K}-Means Notation}

We denote the distance of a point to a set of means $M$ by $d(x,M) = \min_{\mu\in M}\{\norm{x-\mu}\}$ and the $K$-means cost by $\km(X,w,M) = \sum_{x\in X} w(x)d(x,M)^2$.
Let $C\subseteq X$ be some cluster, then $\km(C,w) = \sum_{x\in C} w(x) \norm{x-\mu_w(C)}^2$, where $\mu_w(C) = \sum_{x\in C}w(x)x/w(C)$.

  \section{Coresets for Fuzzy \texorpdfstring{$K$}{K}-Means}\label{sec:coresets}

A coreset is a representation of a data set that preserves properties of the original data set \cite{harpeled03}.
Formally, we require the cost of a set of means with respect to the coreset to be close to the cost the same set of means incurs on the original data.

\begin{definition}[Coreset]\label{def:coreset}
  Let $\epsilon\in(0,1)$.
  A set $S\subset\R^D$ together with a weight function $\wfuncN{w_S}{S}$ is called an $\epsilon$-coreset of $(X,w)$ for the fuzzy $K$-means problem if
  \begin{align}
    \forall M\subset\R^D, \abs{M}\leq K :\  \phi(S,w_s,M) \in [1\pm\epsilon] \phi(X,w,M) \ ,\label{eq:def:coreset}
  \end{align}
  We sometimes refer to a coreset as a \emph{strong coreset}.
\end{definition}

In the following, we show how to construct coresets for the fuzzy $K$-means problem with high probability.
To this end, our proof consists of two independent steps.
First, we show that it is sufficient to construct a so-called weak coreset \cite{Feldman07} for the fuzzy $K$-means problem fulfilling certain properties.
Second, we present an adaptation of Chen's coreset construction for $K$-means \cite{Chen09} which computes weaks coresets with the desired properties, with high probability.

\begin{theorem}\label{thm:coreset}
  There is an algorithm that, given a set $X\subset \R^D$, $K\in\N$, $\delta\in(0,1)$, and $\epsilon\in(0,1)$, computes an $\epsilon$-coreset $(S,w_s)$, with $S\subseteq X$ and $\wfuncN{w_S}{S}$, of $(X,\one)$ for the fuzzy $K$-means problem, with probability at least $1-\delta$, such that
    \[ \abs{S} \in \cO\left( \log(N)\log(\log(N))^2  \epsilon^{-3}  D K^{4m-1} \log(\delta^{-1})\right) \ . \]
  The algorithms' runtime is $\cO(NDK\log(\delta^{-1})+\abs{S})$.
\end{theorem}

This result trivially generalizes to integer weighted data sets, by treating each point $x\in X$ as $w(x)$ copies of the same point.
However, in that case we have to replace each occurrence on $N$ in the runtime of the algorithm and the size of the coreset by $w(X)$.
For rational weights, we normalize the weight function.
This incurs an additional multiplicative factor of $\wmax(X)/\wmin(X)$ to each occurrence of $N$.

\subsection{From Weak to Strong Coresets}\label{subsec:weaktostrong}

Weak coresets are a relaxation of the previously introduced (strong) coresets.
Consider a set of points together with a weight function and a set of solutions.
This forms a weak coreset if the set of solutions contains a solution close to the optimum and the coreset property \eqref{eq:def:coreset} is satisfied for all solutions from the solution set.

\begin{definition}[Weak Coresets]\label{def:weakcoreset}
 A set $S\subset\R^D$ together with a weight function $\wfuncN{w_S}{S}$ and a set of solutions $\Theta \subseteq \{\theta \mid \theta\subset\R^D, \abs{\theta} \leq K\}$ is called a weak $\epsilon$-coreset of $(X,w)$ for the fuzzy $K$-means problem if
 \begin{align*}
 	&\exists M\in\Theta:\ \phi(S,w_s,M) \leq (1+\epsilon)\cdot \phi^{opt}(X,w) \text{ and} \\
 	&\forall M\in\Theta:\ \phi(S,w_s,M) \in [1\pm\epsilon] \phi(X,w,M) \ .
 \end{align*}
\end{definition}

In contrast to the definition of weak coresets for the $K$-means problem \cite{Feldman07}, we consider elements $M$ of a given set of solutions $\Theta$ instead of subsets of a set of candidate means.
This is just a slight generalization which allows us to characterize solutions more precisely.

One difficulty when analysing the fuzzy $K$-means objective function is that, in optimal solutions, clusters are never empty.
Consider a set of means, where there exists a mean which is far away from every point.
In an optimal hard clustering, this mean's cluster is empty and we can safely ignore it in the analysis.
For fuzzy $K$-means, this is not the case.
In an optimal solution, every point has a non-trivial membership to this mean, thus it cannot be ignored (or removed from the solution) without increasing the cost.
Bounding the cost of means with small membership mass proves to be rather difficult.
A central concept we use to control the cost of such means are fuzzy clusters which are almost empty, or negligible.

\begin{definition}[negligible]\label{def:negligible}
	Let $M\subset \R^D$ with $\abs{M} \leq K$.
	We say the $k^{th}$ cluster of a membership function $r: X\times [\abs{M}]\rightarrow [0,1]$ is $(K,\epsilon)$-negligible if $\forall x\in X:\ r(x,k) \leq \epsilon/(4mK^2)$.
	In the following, we omit the parameters $(K,\epsilon)$ if they are clear from context.
\end{definition}

We cannot preclude the possibility that an optimal fuzzy $K$-means clustering contains a negligible cluster.
However, we can circumvent negligible clusters altogether, by observing that we can remove a mean inducing a negligible cluster without increasing the cost significantly.

\begin{theorem}[\cite{bbb16}]\label{thm:remove-clusters}
	Let $M\subset \R^D$ with $\abs{M} \leq K$ and $\epsilon\in(0,1)$.
	There exists a set of means $M'\subseteq M$ with $\phi(X,w,M') \leq (1+\epsilon) \phi(X,w,M)$,
	such that the optimal membership function with respect to $M'$ contain no negligible clusters.
\end{theorem}

Given some set of means, the optimal memberships of a point depend only on the location of the point relative to the means and not on its weight or any other points in the data set \cite{bezdek84}.
This means that negligible clusters are, in some sense, transitive.
That is: If a cluster induced by some set of means is negligible, then it is also negligible with respect to any subset of $X$ and the same set of means.
Using this observation we can prove our key weak-to-strong result.

\begin{lemma}[weak-to-strong]\label{lem:weaktostrong}
	 Let $\epsilon\in(0,1)$ and
   $\TnegX$ be the set of all sets of at most $K$ means inducing no negligible cluster with respect to $X$.
	 If $S\subseteq X$ and $\wfuncN{w_S}{S}$, such that $(S,w_s,\TnegX)$ is weak $\epsilon$-coreset of $(X,w)$ for the fuzzy $K$-means problem, then $(S,w_S)$ is a strong $(3\epsilon)$-coreset of $(X,w)$ for the fuzzy $K$-means problem.
\end{lemma}

\begin{proof}
  We need to verify that the coreset property \eqref{eq:def:coreset} holds for all solutions $M\subset \R^D$ with $\abs{M}\leq K$.
  Since $(S,w_s,\TnegX)$ is a weak $\epsilon$-coreset we only have to show this for all $M\not\in\TnegX$.
	From \Cref{thm:remove-clusters}, we know that there exists $M'\in\TnegX$, $M'\subseteq M$ with $\phi(X,w,M') \leq (1+\epsilon)\phi(X,w,M)$.

	We obtain the upper bound by observing that
	\begin{align*}
		\phi(S,w_S,M) &\leq \phi(S,w_S,M') \tag{$M'\subseteq M$} \\
									&\leq (1+\epsilon) \phi(X,w,M') \tag{weak coreset property} \\
									&\leq (1+\epsilon)^2 \phi(X,w,M) \tag{choice of $M'$} \\
									&\leq (1+3\epsilon) \phi(X,w,M) \ . \tag{$\epsilon\in(0,1)$}
	\end{align*}

	The lower bound is slightly more involved.
	Again, from \Cref{thm:remove-clusters}, we obtain that there exists $M'_S\in\TnegS$, $M'_S\subseteq M$ with $\phi(S,w_S,M'_S) \leq (1+\epsilon)\phi(S,w_S,M)$.
	Recall that for each point, the membership induced by some set of means only depends on the point itself and the given set of means.
	In particular, this membership does not depend on the weight of the point, nor on other data points.
	Hence, if there is no point in $X$ such that the induced membership with respect to some mean $\mu_k\in M$ is larger than some constant, then there is no point in $S\subseteq X$, such that the induced membership to $\mu_k\in M$ is larger than this constant.
	Since $M'\in\TnegX$, it holds that all means in $M\setminus M'$ induce negligible clusters on $S$ and thus $M'_S\subseteq M'$.
	We can conclude
	\begin{align*}
		\phi(S,w_S,M) &\geq \frac{1}{1+\epsilon}\phi(S,w_S,M'_S) \tag{choice of $M'_S$}  \\
									&\geq \frac{1}{1+\epsilon}\phi(S,w_S,M') \tag{$M'_S\subseteq M'$} \\
									&\geq \frac{1-\epsilon}{1+\epsilon}\phi(X,w,M') \tag{weak coreset property} \\
									&\geq \frac{1-\epsilon}{1+\epsilon}\phi(X,w,M) \tag{$M'\subseteq M$} \\
									&\geq (1 - 3\epsilon)\phi(X,w,M) \ . \tag{$\epsilon \geq 0$}
	\end{align*}
\end{proof}

\subsection{Weak Coresets for Solutions with Non-Negligible Clusters}

We adapt Chen's coreset construction for the $K$-means problem \cite{Chen09} to construct a set $S\subseteq X$ and weight function $\wfuncN{w_S}{S}$ such that $(S,w_s,\TnegX)$ is a weak $\epsilon$-coreset of $(X,\one)$ for the fuzzy $K$-means problem.
Applying \Cref{lem:weaktostrong} to this construction yields \Cref{thm:coreset}.

\begin{lemma}\label{lem:weak-coreset}
	There is an algorithm, which computes $S\subseteq X$ and $\wfuncN{w_S}{S}$ such that $(S,w_S,\TnegX)$ is a weak $\epsilon$-coreset of $(X,\one)$ for the fuzzy $K$-means problem, with high probability.
\end{lemma}

Along the lines of Chen's original proof we first show how to construct a set fulfilling the coreset property for a finite amount of solutions.

\begin{algorithm}[ht]
  \SetEndCharOfAlgoLine{}
  \caption{\textsc{Chen's Sampling}}\label{alg:chen}
	\KwIn{$X\subset\R^D$,  $K\in\N$, $\gamma\in\N$, $\alpha,\beta\in\R_{\geq 1}$, an $(\alpha,\beta)$-bicriteria approximation $A\subset\R^D$ of $K$-means on $X$, $\delta\in(0,1)$, $\epsilon\in(0,1)$}
  Let $A_1,\dots,A_{\abs{A}}\subseteq X$ be an $\abs{A}$-means partition induced by $A$.\;
  $F \gets \lceil \frac{1}{2}\log(\alpha N)\rceil$\;
  $R \gets \sqrt{\frac{\km(X,\one,A)}{\alpha N}}$. \;
  $q \gets \mathbf{q}\cdot \left(\frac{\alpha K^{m-1}}{\epsilon}\right)^2 \ln\left(\frac{4\beta K F \gamma^K}{\delta}\right)$ for a sufficiently large constant $\mathbf{q}$\;
	\For{$\OneTo{k}{\beta K}$ and $\OneTo{j}{F}_0$}
  {
    \[ L_{k,j} \gets \begin{cases} \ball(a_k,R) & \mbox{if $j = 0$} \\ \ball(a_k,2^jR) \setminus \ball(a_k,2^{j-1}R) & \mbox{if $j \geq 1$.} \end{cases} \]
    where $\ball(c,r) = \{x\in\R^D|\norm{x-c}\leq r\}$. \;
    $X_{k,j} \gets L_{k,j}\cap A_k$\;
    \If{$X_{k,j}\neq \emptyset$}
    {
      $S_{k,j}\gets\emptyset$\;
      \For{$\OneTo{i}{q}$}
      {
        Sample $x$ uniformly at random from $X_{k,j}$\;
        \If{$x\not\in S_{k,j}$}
        {
          $S_{k,j} \gets S_{k,j}\cup \{x\}$\;
          $w_S(x) \gets \frac{\abs{X_{k,j}}}{q}$\;
        }
        \Else
        {
          $w_S(x) \gets w_S(x) + \frac{\abs{X_{k,j}}}{q}$\;
        }
      }
    }
  }
  \Return{$\left( \bigcup_{\OneTo{k}{\beta K}, \OneTo{j}{F}_0} S_{k,j},w_S\right)$}\;
\end{algorithm}

\begin{lemma}\label{thm:weaker-coresets}
  For each $\Gamma\subset\R^D$ with $\abs{\Gamma} \leq \gamma$ the output of \Cref{alg:chen} satisfies $S\subseteq X$, $w_S:S\rightarrow \N$, $\sum_{s\in S} w_S(s) = N$, and, with probability $1-\delta$,
  \[ \forall M\subseteq \Gamma, \abs{M}\leq K:
  \begin{array}{rl}
    \phi(S,w_S,M) &\in [1\pm\epsilon] \phi(X,\one,M) \\
    \km(S,w_S,M) &\in [1\pm\epsilon/K^{m-1}]\km(X,\one,M) \\
  \end{array} \ .  \]
\end{lemma}

\begin{proof}
  Observe, that the $X_{k,j}$ form a partition of $X$.
  Since $X_{k,j}\subseteq A_k\subseteq X$, the $X_{k,j}$ are pairwise disjoint subsets of $X$.
  It remains to show that $X\subseteq \bigcup_{k,j} X_{k,j}$.
  If $y\notin\bigcup_{k,j} X_{k,j}$, then
  \begin{align*}
    d(y,A) & > 2^F R = \sqrt{\alpha N}\cdot \sqrt{\frac{\km(X,\one,A)}{\alpha N}} = \sqrt{\km(X,\one,A)} \ .
  \end{align*}
  Since for all $x\in X$ we have $d(x,A)^2 \leq \km(X,\one,A)$, $y$ can not be in $X$.

  Further, notice that for each $\OneTo{k}{\beta K}$, $\OneTo{j}{F}_0$ we have $\sum_{s\in S_{k,j}}w_S(s) = \abs{X_{k,j}}$.

  By our description of \Cref{alg:chen}, the weights $w_S$ are not necessarily natural numbers.
  In the original proof \cite{Chen09} Chen shows to solve this problem.

  Fix an arbitrary $\Gamma \subset\R^D$ with $\abs{\Gamma}\leq \gamma$.
  Consider an arbitrary but fixed $M\subseteq\Gamma$ with $\abs{M}\leq K$.
  By the triangle inequality, and since the $X_{k,j}$ form a partition of $X$, we obtain

  \[ \abs{\phi(X,\one,M) - \phi(S,w_S,M)} \leq \sum_{k=1}^{\beta K}\sum_{j=0}^F \abs{\phi(X_{k,j},\one,M) - \phi(S_{k,j},w_S,M)} \ . \]

  Fix some $\OneTo{k}{\beta K}$, $\OneTo{j}{F}_0$ and consider the summand $\abs{\phi(X_{k,j},\one,M) - \phi(S_{k,j},w_S,M)}$.
  By definition of the weights, we have
  \begin{align*}
    \frac{1}{\abs{X_{k,j}}}  \phi(S_{k,j},w_S,M)
    &= \frac{1}{\abs{X_{k,j}}} \sum_{s\in S_{k,j}} \phi(\{s\},w_S,M) \\
    &= \frac{1}{q} \sum_{s\in S_{k,j}} \phi(\{s\}, \one, M) \ ,
  \end{align*}

  and thus, we can write
  \begin{align*}
    \abs{\phi(X,\one,M) - \phi(S,w_S,M)}
    &= \abs{X_{k,j}}\abs{\frac{1}{\abs{X_{k,j}}} \phi(X_{k,j}, \one, M) - \frac{1}{q} \phi(S_{k,j}, \one, M)} \ .
  \end{align*}

  Let $\tm_{k,j} = \arg\min_{x\in X_{k,j}}\{d(x,M)\}$, $\epsilon' := \epsilon/(44\alpha K^{m-1})$, and $\delta' := \delta / (2 \beta K R (\gamma+1)^K )$.
  Applying a concentration bound due to Haussler \cite{haussler92} yields that with probability at least $1- \delta'$, we have
  \[ \abs{\phi(X_{k,j}, \one, M) - \phi(S_{k,j},w_S,M)} \leq 4 \epsilon' \left( \abs{X_{k,j}} d(\tm_{k,j},M)^2 + \abs{X_{k,j}}  2^{2j+1}R^2 \right) \ . \]
  We bound the two summands separately.

  For the first term we straightforwardly observe.
  \[ \abs{X_{k,j}}d(\tm_{k,j},M)^2 \leq \sum_{x\in X_{k,j}} d(x,M)^2 = \km(X_{k,j},\one,M) \ . \]

  Next, consider the second term.
  For $j=0$, we know
  \[ \abs{X_{k,j}}2^{2j+1}R^2 = \abs{X_{k,j}}2R^2  = \frac{2}{\alpha}\frac{\abs{X_{k,j}}}{N}\km(X,\one,A) \leq 2 \frac{\abs{X_{k,j}}}{N}\km(X,\one,A) \ .\]
  For $j\geq 1$, recall that $X_{k,j}\subseteq L_{k,j}$.
  Hence, for all $x\in X_{k,j}$, we have $2^{2j-2} R^2 \leq \norm{x-a_k}^2 = d(x,A)^2$ and thus
  \[ \abs{X_{k,j}}2^{2j+1}R^2 \leq 8\sum_{x\in X_{k,j}} \norm{x-a_k}^2 = 8\km(X_{k,j},\one,A) \ . \]

  Putting this together we obtain
  \begin{align*}
    \abs{\phi(X_{k,j}, \one, M) - \phi(S_{k,j},w_S,M)} \leq 4 \epsilon' (&\km(X_{k,j},\one,M) + \\
      &8\km(X_{k,j},\one,A) + 2 \frac{\abs{X_{k,j}}}{N}\km(X,\one,A) ) \ .
  \end{align*}

  Note that for $X_{k,j} = \emptyset$ we have $\phi(X_{k,j}, \one, M) = \phi(S_{k,j}, w_S, M) = 0$.

  By using the union bound, we know that this probabilistic upper bound holds simultaneously for every $\OneTo{k}{\beta K}$ and $\OneTo{j}{F}_0$ with probability at least $1-\delta/(2 \gamma^K)$.
  Recall, that the $X_{k,j}$ form a partition of $X$.
  Hence, by taking the sum on both sides we obtain
  \[ \sum_{k=1}^{\beta K} \sum_{j=0}^F \abs{\phi(X_{k,j}, \one, M) - \phi(S_{k,j}, w_S, M)} \leq 4\epsilon' \left( \km(X,\one,M) + 10\km(X,\one,A) \right) \ . \]
  Recall, that $\abs{M}\leq K$,
  \begin{align*}
    \km(X,\one,M) &\leq K^{m-1}\phi(X,\one,M) \mbox{ , and}\\
    \km(X,\one,A) &\leq \alpha \km(X,\one,M) \leq \alpha K^{m-1}\phi(X,\one,M) \ .
  \end{align*}
  We can conclude
  \[ \abs{\phi(X,\one,M) - \phi(S,w_S,M)} \leq  4\epsilon'\cdot K^{m-1}(1+ 10\alpha)\phi(X,\one,M) \leq  \epsilon \cdot \phi(X,\one,M) \ , \]
  where the last inequality is by definition of $\epsilon'$ and $\alpha$.

  Note that there are $\abs{\Gamma}^K\leq \gamma^K$ different sets $M\subseteq\Gamma$ with $\abs{M}\leq K$.
  Thus, by union bound, our upper bound holds simultaneously for all such $M$ with probability  at least $1-\delta/2$.

  Following the same line of arguments as before, we obtain that, also with probability $1-\delta/2$ we have for all $M\subseteq\Gamma$ with $\abs{M}\leq K$,
  \[ \abs{\km(X,\one,M) - \km(S,w_S,M)} \leq \epsilon/K^{m-1} \km(X,\one,M) \ . \]

  Finally, using the union bound once more, we combine these two results to conclude the proof.
\end{proof}

Before we start working towards the proof of \Cref{lem:weak-coreset}, we formulate a technical observation which we use to compare the fuzzy cost of equally sized data sets.

\begin{lemma}\label{lem:cost-comp}
  Let $X,Y,M \subset\R^D$ with $\abs{X} = \abs{Y} = N$.
  For all $\epsilon\in[0,1]$ we have
  \[ \abs{\phi(X,\one,M) - \phi(Y,\one,M)} \leq \left(1+\frac{1}{\epsilon}\right)\sum_{\OneTo{n}{N}} \norm{x_n-y_n}^2 + \epsilon\cdot\min\{\phi(X,\one,M),\phi(Y,\one,M) \} \ . \]
\end{lemma}

\begin{proof}
  Let $r_X,r_Y$ be optimal memberships induced by $M$ on $X$ and $Y$, respectively.
  We distinguish two cases.
  First, if $\phi(X,\one,M) \geq \phi(Y,\one,M)$, then
  \begin{align*}
    \phi(X,\one,M) - \phi(Y,\one,M) \leq &\sum_{\OneTo{n}{N}} \sum_{\OneTo{k}{\abs{M}}} r_y(y_n,k)^m (\norm{x_n-\mu_k}^2 - \norm{y_n-\mu_k}^2) \\
    \leq &\sum_{\OneTo{n}{N}} \sum_{\OneTo{k}{\abs{M}}} r_y(y_n,k)^m (\norm{x_n-y_n}^2 + 2 \norm{x_n-y_n}\norm{y_n-\mu_k}) \\
    = &\sum_{\OneTo{n}{N}} \norm{x_n-y_n}^2 (\sum_{\OneTo{k}{\abs{M}}} r_y(y_n,k)^m)  \\
    &+ 2 \sum_{\OneTo{n}{N}} \sum_{\OneTo{k}{\abs{M}}} \norm{x_n-y_n}\norm{y_n-\mu_k} \ .
  \end{align*}
  Observe, that for all $a\in\R_+$ and $x,y\in\R$ we have $0\leq (ax+y/a)^2 = a^2x^2 - 2xy + y^2/a^2$ and hence $2xy \leq a^2x^2 + y^2/a^2$.
  Thus, we can bound
  \begin{align*}
    2 \sum_{\OneTo{n}{N}} \sum_{\OneTo{k}{\abs{M}}} \norm{x_n-y_n}\norm{y_n-\mu_k} \leq &\sum_{\OneTo{n}{N}} \sum_{\OneTo{k}{\abs{M}}} \left( \frac{1}{\epsilon}\norm{x_n-y_n}^2 + \epsilon\norm{y_n-\mu_k}^2\right) \\
    = &\epsilon\phi(Y,\one,M) \\
    &+ \frac{1}{\epsilon}\sum_{\OneTo{n}{N}} \norm{x_n-y_n}^2 (\sum_{\OneTo{k}{\abs{M}}} r_y(y_n,k)^m) \ .
  \end{align*}
  Recall, that $r_Y$ is a membership function, thus for each $y\in Y$ we have $\sum_{\OneTo{k}{\abs{M}}}r(y,k)^m \leq 1$.
  We can conclude
  \begin{align*}
    \phi(X,\one,M) - \phi(Y,\one,M) &\leq \left(1+\frac{1}{\epsilon}\right)\sum_{\OneTo{n}{N}} \norm{x_n-y_n}^2 + \epsilon\phi(Y,\one,M) \ .
  \end{align*}

  Second, if $\phi(X,\one,M) < \phi(Y,\one,M)$, then we obtain
  \[ \phi(X,\one,M) - \phi(Y,\one,M) \leq \left(1+\frac{1}{\epsilon}\right)\sum_{\OneTo{n}{N}} \norm{x_n-y_n}^2 + \epsilon\phi(X,\one,M) \]
   analogously.
\end{proof}

\begin{algorithm}[ht]
  \SetEndCharOfAlgoLine{}
  \caption{\textsc{Fuzzy $K$-Means Coreset}}\label{alg:fuzzy-coresets}
  \KwIn{$X\subset \R^D$, $K\in\N$, $\delta\in (0,1)$, $\epsilon\in (0,1)$}
  Apply the algorithm from \cite{aggarwal09} which computes, with probability $1-\delta/3$ an $(\alpha,\beta)$-bicriteria approximation $A\subset\R^D$ of $K$-means on $X$\;
  \[ \teps \gets \frac{\epsilon}{\mathbf{a}\alpha K^{m-1}}  \] for a sufficiently large constant $\mathbf{a}$\;
  \[ \gamma \gets \beta K \left( \frac{1}{2}\log\left( \frac{\mathbf{b}\alpha N}{\teps^2(\epsilon/(4mK^2))^m} \right) + 1 \right) \left(\frac{\mathbf{c}}{\teps}\right)^D \] for sufficiently large constants $\mathbf{b,c}$\;
   $(S,w_S) \gets$ \Cref{alg:chen} with $X$, $K$, $\gamma$, $\alpha$, $\beta$, $\delta/3$, $\teps$ \;
  \Return{$(S,w_S)$}\;
\end{algorithm}

Next, we prove \Cref{lem:weak-coreset} by analysing a single run of \Cref{alg:fuzzy-coresets}.
Assume Step~1 succeeded, i.e.\ $\abs{A}\leq \beta K$ and $\km(X,\one,A) \leq \alpha \km^{opt}(X,\one)$.
Fix some solution $M\in \TnegX$, i.e.\ $\abs{M} \leq K$ and $\forall k\in\abs{M} \exists x\in X: r(x,k) > \epsilon/(4mK^2)$.

\paragraph{Search Space}
We define a search space $\cU$ consisting of large balls around the means in $A$.
Our analysis afterwards distinguishes two cases: Is $M$ contained $\cU$, or not.
Let
\[ E = \left\lfloor\frac{1}{2}\log\left( \frac{\mathbf{b}\alpha N}{\teps^2(\epsilon/(4mK^2))^m} \right)\right\rfloor \mbox{ and } \cU = \bigcup_{a\in A} \ball (a,2^E R) \ . \]
We observe that the search space $\cU$ covers large balls around the points of $X$, as well.
\begin{lemma}\label{lem:u-coverage}
  \[ \bigcup_{x\in X} \ball(x,r)\subseteq \cU \mbox{ where } r = \frac{1}{2}\sqrt{\frac{\mathbf{b} \km(X,\one,A)}{\teps^2(\epsilon/(4mK^2))^m}} \  . \]
\end{lemma}
\begin{proof}
  Assume there exists $x\in X$ with $\ball(x,r)\not\subseteq \cU$.
  Hence, for all $a\in A$ we have $\ball(x,r)\not\subseteq \ball(a,2^E R)$ and thus
  \[ d(x,A) \geq 2^E R - r \ . \]
  Observe that
  \[ 2^E R \leq \sqrt{\frac{\mathbf{b}\km(X,\one,A)}{\teps^2(\epsilon/(4mK^2))^m}} \]
  Hence,
  \begin{align*}
    d(x,A) \geq 2^E R - r &= \sqrt{\frac{\mathbf{b}\km(X,\one,A)}{\teps^2(\epsilon/(4mK^2))^m}} - \frac{1}{2}\sqrt{\frac{\mathbf{b} \km(X,\one,A)}{\teps^2(\epsilon/(4mK^2))^m}} \\
    &\geq \sqrt{\mathbf{b}\km(X,\one,A)} - \frac{1}{2}\sqrt{\mathbf{b} \km(X,\one,A)} \tag{$\teps^2(\epsilon/(4mK^2))^m \leq 1$} \\
    &> \sqrt{\km(X,\one,A)} \ . \tag{for $\mathbf{b}$ large enough}
  \end{align*}
  This is a contradiction to the definition of $K$-means cost.
\end{proof}

Next, we define a function $g$ which discretizes $\cU$, i.e.\ a function $g:\cU\rightarrow\cU$ with a finite image.
Similar to before, let
\[ \cU_{k,j} = \begin{cases} \ball(a_k,R) & \mbox{if $j= 0$} \\ \ball(a_k,2^j R)\setminus \ball(a_k 2^{j-1} R) & \mbox {if $j \geq 1$} \end{cases} \ , \]
for all $\OneTo{j}{E}$, hence $\bigcup_{k,j} \cU_{k,j} = \cU$.
Assume that each set $\cU_{k,j}$ is partitioned into cells via an axis-parallel grid with side length
\[ \teps\frac{2^j R}{\sqrt{D}} \ . \]
For each cell pick some representative inside the cell.
Let $G$ be some function mapping each point $x\in X$ to a cell containing $x$, i.e.\ $x\in G(x)\subseteq \cU_{k,j}$ for some $k$ and $j$, where ties are broken arbitrarily.
Finally, we set $g(x)$ to be the representative of $G(x)$.

Let $G = \{ g(u) \mid u\in \cU\}$.
Using the volume argument presented in \cite{Chen09}, one can prove that $\abs{G} \leq \abs{A}\cdot(E+1)\cdot(\mathbf{c}/\teps)^D = \gamma$.
By \Cref{thm:weaker-coresets} we have that for result $(S,w_S)$ of \Cref{alg:chen} we have
\[ \forall M\subseteq G, \abs{M}\leq K:
\begin{array}{rl}
  \phi(S,w_S,M) &\in [1\pm\teps] \phi(X,\one,M) \\
  \km(S,w_S,M) &\in [1\pm\teps/K^{m-1}]\km(X,\one,M) \\
\end{array} \ ,  \]
with probability at least $1-\delta/3$.
In the following, we assume \Cref{alg:chen} was successful.
Furthermore, we can apply the proof presented in \cite{Chen09} to conclude that $(S,w_S)$ is a strong $\teps/K^{m-1}$-coreset for the $K$-means problem.

\paragraph{Closeness of a Point and Its Representative}

By definition, a point $u\in\cU$ and its representative $g(u)$ are contained in the same grid cell.
Hence, we can bound their distance in terms of their $K$-means cost with respect to $A$.

\begin{lemma}\label{lem:closeness}
  For $u\in\cU$ and $y\in\R^D$ we have
  \begin{align*}
    \norm{u-g(u)} &\leq 2\teps(\min\{d(u,A),d(g(u),A)\} + R) \\
      &\leq 2\teps(\min\{\norm{y-u},\norm{y-g(u)}\} + d(y,A) + R)
  \end{align*}
  and
  \[ \norm{u-g(u)}^2 \leq 12\teps^2 (\min\{\norm{y-u},\norm{y-g(u)}\}^2 + d(y,A)^2 + R^2) \ . \]
\end{lemma}
\begin{proof}
  Let $\cU_{k,j}$ be the cell containing both $u$ and $g(u)$.
  If $j = 0$, then $\norm{u-g(u)} \leq \teps R$.
  If $j \geq 1$, then $\norm{u-g(u)} \leq \teps 2^j R$ and $\min\{d(u,A),d(g(u),A)\} \geq 2^{j-1} R$.
  By simply taking the sum over both inequalities we obtain
  \[ \norm{u-g(u)} \leq \teps 2^j R + \teps R \leq 2\teps (\min\{d(u,A),d(g(u),A)\} + R) \ . \]
  Furthermore, by triangle inequality we have
  \begin{align*}
    \min\{d(u,A),d(g(u),A)\} &\leq \min\{\norm{u-y}+d(y,A),\norm{g(u)-y} + d(y,A)\} \\
      &= \min\{\norm{u-y},\norm{g(u)-y}\} + d(y,A) \ .
  \end{align*}
  We obtain the third inequality by recalling that $\forall a,b,c\in\R: (a+b+c)^2 \leq 3(a^2+b^2+c^2)$.
\end{proof}

\paragraph{Replacing Means by Their Representatives}

Consider some set of means $M\subseteq \cU$ with $\abs{M} \leq K$.
Our next goal is to use the previous lemma to compare the $K$-means cost of $X$ with respect to $M$ and the set $g(M) = \{ g(\mu) \mid \mu\in M\}$.

\begin{lemma}\label{lem:close-k-means}
  \begin{align*}
    d(x,g(M))^2 &\leq d(x,M)^2 + 18\teps(2 d(x,M)^2 + d(x,A)^2 + R^2) \ , \\
    \km(X,\one,g(M)) &\leq \km(X,\one,M) + 18\teps(2\km(X,\one,M) + \km(X,\one,A) + NR^2) \ , \\
    d(x,M)^2 &\leq d(x,g(M)) + 18\teps(2 d(x,g(M))^2 + d(x,A)^2 + R^2) \mbox{ , and} \\
    \km(X,\one,M) &\leq \km(X,\one,g(M)) + 18\teps(2\km(X,\one,g(M)) + \km(X,\one,A) + NR^2) \ .
  \end{align*}
\end{lemma}
\begin{proof}
  Let $x\in X$ and $\mu\in M$ such that $\norm{x-\mu} = d(x,M)$.
  Observe that
  \begin{align*}
    d(x,g(M))^2 - d(x,M)^2 &\leq \norm{x-g(\mu)}^2 - \norm{x-\mu}^2 \\
      &\leq \norm{\mu-g(\mu)}^2 + 2\norm{\mu-g(\mu)}\norm{x-\mu} \ .
  \end{align*}
  By \Cref{lem:closeness} we obtain
  \begin{align*}
    \norm{\mu-g(\mu)}^2 &\leq 12\teps^2 (\min\{\norm{x-\mu},\norm{x-g(\mu)}\}^2 + d(x,A)^2 + R^2) \\
      &\leq 12\teps^2 (d(x,M)^2 + d(x,A)^2 + R^2) \ ,
  \end{align*}
  and
  \begin{align*}
    2\norm{\mu-g(\mu)}\norm{x-\mu} &\leq 4\teps(\min\{\norm{x-\mu},\norm{x-g(\mu)}\} + d(x,A) + R) \norm{x-\mu} \\
      &\leq 2\teps(d(x,M) + d(x,A) + R)^2 + \norm{x-\mu}^2) \\
      &\leq 6\teps(2d(x,M)^2+d(x,A)^2 + R^2) \ .
  \end{align*}
  Taking the sum over all points in $X$ and for each point the sum of the two upper bounds yields the first part of the claim.
  The second part can be achieved analogously.
\end{proof}

We obtain a similar result with respect to the fuzzy $K$-means cost.

\begin{lemma}\label{lem:close-fuzzy-k-means}
  \[ \abs{\phi(X,\one,M) - \phi(X,\one,g(M))} \leq 32\teps (\km(X,\one, M) + \km(X,\one,A) + NR^2) \ . \]
\end{lemma}
\begin{proof}
  By triangle inequality we have
  \[ \abs{\phi(X,\one,M) - \phi(X,\one,g(M))} \leq \sum_{x\in X} \abs{\phi(\{x\},\one,M) - \phi(\{x\},\one,g(M))} \ . \]
  Fix some $x\in X$, let $r$ and $r_g$ be optimal membership functions with respect to $M$ and $g(M)$, and denote $\cE = \abs{\phi(\{x\},\one,M) - \phi(\{x\},\one,g(M))}$.

  If $\phi(\{x\},\one,M) \geq \phi(\{x\},\one,g(M))$, then
  \begin{align*}
    \cE &= \sum_{\OneTo{k}{\abs{M}}}r(x,k)^m\norm{x-\mu_k}^2 - \sum_{\OneTo{k}{\abs{M}}}r_g(x,k)^m\norm{x-g(\mu_k)}^2 \\
    &\leq \sum_{\OneTo{k}{\abs{M}}}r_g(x,k)^m (\norm{x-\mu_k}^2 - \norm{x-g(\mu_k)}^2) \\
    &\leq \sum_{\OneTo{k}{\abs{M}}}r_g(x,k)^m (\norm{\mu_k-g(\mu_k)}^2 + 2\norm{\mu_k-g(\mu_k)}\norm{x-g(\mu_k)}) \ .
  \end{align*}
  Otherwise, we obtain analogously
  \[ \cE \leq \sum_{\OneTo{k}{\abs{M}}}r(x,k)^m (\norm{\mu_k-g(\mu_k)}^2 + 2\norm{\mu_k-g(\mu_k)}\norm{x-\mu_k}) \ . \]
  In the following, we obtain the claim as the larger of the upper bound on these two terms.
  For each term, we derive upper bounds of each summand.
  \begin{align*}
    \sum_{\OneTo{k}{\abs{M}}}r_g(x,k)^m \norm{\mu_k-g(\mu_k)}^2 &\leq 12\teps^2 \sum_{\OneTo{k}{\abs{M}}}r_g(x,k)^m (\norm{x-g(\mu_k)}^2+d(x,A)^2 + R^2) \tag{\Cref{lem:closeness}} \\
    &\leq 12\teps^2 \left( \sum_{\OneTo{k}{\abs{M}}}r_g(x,k)^m\norm{x-g(\mu_k)}^2+d(x,A)^2 + R^2\right) \tag{$\sum_{\OneTo{k}{\abs{M}}}r_g(x,k)^m \leq 1$} \\
    &\leq 12\teps^2 (d(x,g(M))^2+d(x,A)^2 + R^2) \\
    &\leq 24\teps^2 (d(x,M)^2+d(x,A)^2 + R^2) \ , \tag{\Cref{lem:close-k-means}}
  \end{align*}
  and similarly
  \begin{align*}
    \sum_{\OneTo{k}{\abs{M}}}r(x,k)^m \norm{\mu_k-g(\mu_k)}^2 &\leq 12\teps^2 \sum_{\OneTo{k}{\abs{M}}}r(x,k)^m (\norm{x-\mu_k}^2+d(x,A)^2 + R^2) \\
    &\leq 12\teps^2 \left( \sum_{\OneTo{k}{\abs{M}}}r(x,k)^m\norm{x-\mu_k}^2+d(x,A)^2 + R^2\right) \\
    &\leq 12\teps^2 (d(x,M)^2+d(x,A)^2 + R^2) \ .
  \end{align*}
  For the mixed terms we observe
  \begin{align*}
      &\sum_{\OneTo{k}{\abs{M}}}r_g(x,k)^m \norm{\mu_k-g(\mu_k)}\norm{x-g(\mu_k)} \\
    \leq 2\teps &\sum_{\OneTo{k}{\abs{M}}}r_g(x,k)^m (\norm{x-g(\mu_k)} + d(x,A) + R) \norm{x-g(\mu_k)} \tag{\Cref{lem:closeness}} \\
    \leq \teps &\sum_{\OneTo{k}{\abs{M}}} \left( r_g(x,k)^m (\norm{x-g(\mu_k)} + d(x,A) + R)^2 + \norm{x-g(\mu_k)}^2 \right) \\
    \leq \teps &\sum_{\OneTo{k}{\abs{M}}} r_g(x,k)^m (4\norm{x-g(\mu_k)}^2 + 3d(x,A)^2 + 3R^2 ) \\
    \leq 4\teps &\left( d(x,A)^2 + R^2 + \sum_{\OneTo{k}{\abs{M}}} r_g(x,k)^m \norm{x-g(\mu_k)}^2 \right) \\
    \leq 8\teps &(d(x,M)^2 + d(x,A)^2 + R^2) \ ,
  \end{align*}
  and once again similarly
  \begin{align*}
    &\sum_{\OneTo{k}{\abs{M}}}r(x,k)^m \norm{\mu_k-g(\mu_k)}\norm{x-\mu_k} \\
    \leq 2\teps &\sum_{\OneTo{k}{\abs{M}}}r(x,k)^m (\norm{x-\mu_k} + d(x,A) + R)\norm{x-\mu_k} \\
    \leq \teps &\sum_{\OneTo{k}{\abs{M}}}\left( r(x,k)^m (\norm{x-\mu_k} + d(x,A) + R)^2 + \norm{x-\mu_k}^2 \right) \\
    \leq 4\teps &\left( d(x,A)^2 + R^2 + \sum_{\OneTo{k}{\abs{M}}} r(x,k)^m \norm{x-\mu_k}^2 \right) \\
    \leq 4\teps &( d(x,M)^2 + d(x,A)^2 + R^2)  \ . \\
  \end{align*}
  Taking both sums we conclude
  \begin{align*}
    \cE &\leq \max\{(24+8)\teps^2 (d(x,M)^2+d(x,A)^2 + R^2), (8+4)\teps^2 (d(x,M)^2+d(x,A)^2 + R^2)\} \\
      &\leq 32\teps^2 (d(x,M)^2+d(x,A)^2 + R^2) \ .
  \end{align*}
\end{proof}

\paragraph{Weak Coreset Proof}

Now we have all the ingredients ready to proof that $(S,w_S,\TnegX)$ is a weak $\epsilon$-coreset.
The following two lemmata distinguish two cases, depending on the location of the means in $M$.

\begin{lemma}[Inside the search space]\label{lem:inside-search-space}
  If $M\subseteq \cU$, then
  \[ \abs{\phi(X,\one,M) - \phi(S,w_S,M)} \leq \epsilon\phi(X,\one,M) \ . \]
\end{lemma}
\begin{proof}
  By the triangle inequality we obtain
  \begin{align*}
    &\abs{\phi(X,\one,M) - \phi(S,w_S,M)}  \\
    \leq &\abs{\phi(X,\one,M) - \phi(X,\one,g(M))} + \abs{\phi(X,\one,g(M)) - \phi(S,w_S,g(M))} \\
    + &\abs{\phi(S,w_S,g(M)) - \phi(S,w_S,M)} \ .
  \end{align*}
  Since $M\subseteq \cU$ we can apply \Cref{lem:close-fuzzy-k-means} to the first summand
  \begin{align*}
    \abs{\phi(X,\one,M) - \phi(X,\one,g(M))} &\leq 32\teps (\km(X,\one, M) + \km(X,\one,A) + NR^2) \\
    &\leq 32\teps (\km(X,\one, M) + (1+1/\alpha) \km(X,\one,A)) \tag{By definition of $R$} \\
    &\leq 32(\alpha + 2)\teps \km(X,\one, M) \tag{$A$ is $(\alpha,\beta)$-bicriteria approximation} \\
    &\leq 32(\alpha + 2)K^{m-1}\teps \phi(X,\one, M) \\
    &\leq \epsilon/3 \phi(X,\one,M) \tag{By definition of $\teps$} \ .
  \end{align*}
  Recall that $w_S$ is an integer weight function, such that $\sum_{s\in S}w_S(s) = N$.
  Hence, by replacing each $s\in S$ by $w_S(s)$ copies of the point, we can treat $S$ as a set of size $N$ and with weight function $\one$, and apply \Cref{lem:close-fuzzy-k-means}.
  Further recall, that $(S,w_S)$ is an $\teps/K^{m-1}$-coreset for the $K$-means problem in $X$.
  Thus, we can bound
  \begin{align*}
    \abs{\phi(S,w_S,g(M)) - \phi(S,w_S,M)} \leq &32 \teps (\km(S,w_S, M) + \km(S,w_S,A) + NR^2) \\
    \leq &32 \teps ((1+\teps/K^{m-1})\km(X,\one,M) \\
    &+ (1+\teps/K^{m-1}+1/\alpha) \km(X,\one,A)) \\
    \leq &32 (2\alpha+3) \teps \km(X,\one,M) \\
    \leq &32 (2\alpha+3)K^{m-1} \teps \phi(X,\one,M) \\
    \leq &\epsilon/3 \phi(X,\one,M) \ .
  \end{align*}
  We already assumed that $(S,w_S)$ is an $\teps$-coreset for all solutions from $G$ (and $g(M) \subseteq G$), hence we can directly bound
  \[ \abs{\phi(X,\one,g(M)) - \phi(S,w_S,g(M))} \leq \teps \phi(X,\one,M) \leq \epsilon/3 \phi(X,\one,M) \ . \]
\end{proof}

\begin{lemma}[Outside the search space]\label{lem:outside-search-space}
  If $\exists \mu_k\in M: \mu\not\in \cU$, then
  \[ \abs{\phi(X,\one,M) - \phi(S,w_S,M)} \leq \epsilon\phi(X,\one,M) \ . \]
\end{lemma}
\begin{proof}
  Recall, that the points in $S$ are sampled from the rings $\bigcup_{k,j} X_{k,j} = X$.
  Hence, there is a total function $\ts: X\rightarrow X$, such that $\ts(X_{k,j}) = \{ \ts(x) \mid x\in X_{k,j}\} = S_{k,j}$ and for each $s\in S$ the pre-image has size $\abs{\ts^{-1}(s)} = w_S(s)$.
  By definitions of the rings we know that for each $x\in X$
  \begin{align*}
    \sum_{x\in X} \norm{x-\ts(x)}^2 &\leq 8(\km(X,\one,A) + NR^2) \\
      &\leq 16 \km(X,\one,A) \ .
  \end{align*}
  Furthermore, from \Cref{lem:u-coverage}, we obtain
  \[ d(\mu_k, A)^2 >  \mathbf{b}/(4\teps^2(\epsilon/4mK^2)^m)\km(X,\one,A) \ . \]
  Thus, for all $x\in X$ we have
  \begin{align*}
    \norm{x-\mu_k} &\geq d(\mu_k,A) - \norm{x - a_k} \tag{Where $a_k\in A$ is closest to $\mu_k$.} \\
      &\geq \sqrt{\mathbf{b}/(4\teps^2(\epsilon/4mK^2)^m)\km(X,\one,A)} - \sqrt{\km(X,\one,A)} \\
      &\geq \sqrt{\mathbf{b'}/(\teps^2(\epsilon/4mK^2)^m)}\sqrt{\km(X,\one,A)} \ . \tag{$\mathbf{b'}$ still some sufficiently large constant.}
  \end{align*}
  Let $r$ be optimal memberships with respect to $M$ (which induce no negligible clusters), then
  \begin{align*}
    \phi(X,\one,M) &\geq \sum_{x\in X} r(x,k)^m \norm{x-\mu_k}^2 \\
      &\geq \mathbf{b'}/(4\teps^2(\epsilon/4mK^2)^m)\km(X,\one,A) (\sum_{x\in X} r(x,k)^m) \\
      &\geq \mathbf{b'}/(\teps^2)\km(X,\one,A) \ .
  \end{align*}
  Combining these two observations we obtain
  \[ \sum_{x\in X} \norm{x-\ts(x)}^2 \leq 16 \km(X,\one,A) \leq \teps^2\phi(X,\one,M) \ . \]
  Again, using $\ts$ we can treat $S$ as an unweighted set of size $N$ and thus can apply \Cref{lem:cost-comp}
  \begin{align*}
    \abs{\phi(X,\one,M) - \phi(S,w_S,M)} \leq &\left(1+\frac{1}{\teps}\right)\sum_{x\in X} \norm{x-\ts(x)}^2 \\
    &+ \teps\cdot\min\{\phi(X,\one,M),\phi(S,w_S,M) \} \\
      \leq &(1+\frac{1}{\teps})\teps^2\phi(X,\one,M) + \teps \phi(X,\one,M) \leq \epsilon\phi(X,\one,M) \ .
  \end{align*}
\end{proof}

\paragraph{Size and Runtime}

\begin{lemma}
  For the ouput of \Cref{alg:fuzzy-coresets} we have
  \[ \abs{S} \in \cO\left( \log(N)\log(\log(N))^2  \epsilon^{-3}  D K^{4m-1} \log(\delta^{-1})\right) \ . \]
\end{lemma}
\begin{proof}
  First, we bound the output of \Cref{alg:chen} in terms of $\teps$ and $\gamma$.
  The output set $S$ is the union of $\beta K \cdot F$ sets of size $q$, where $F\in\cO(\log(N))$.
  For the sample size $q$ we have
  \[ q \in \cO\left(\left(\frac{K^{m-1}}{\teps}\right)^2\log\left(\frac{KF\gamma^K}{\delta}\right)\right) \subseteq \cO\left(K^{2m-2}\teps^{-2}K\log(\gamma)\log(\log(N))\log(\delta^{-1})\right) \ . \]
  By our choice of parameters we have $\teps\in \cO(\epsilon/K^{m-1})$ and
  \begin{align*}
    \log(\gamma) &\in \cO\left(\log(K) + \log(\log(N) + \log(\teps^{-1}) + \log((\epsilon/4mK^2)^{-m}) + D\log(\teps^{-1})\right) \\
    &\subseteq \cO\left(\log(K) \log(\log(N)) D\log(K/\epsilon) \log(\log(K/\epsilon)) \right) \\
  \end{align*}
  Overall we obtain
  \begin{align*}
    \abs{S} &\in \cO\left(\log(N)K^{4m-2}\epsilon^{-2}\log(K) \log(\log(N))^2 D\log(K/\epsilon) \log(\log(K/\epsilon))\log(\delta^{-1}) \right)\\
      &\subseteq \cO\left(\log(N)\log(\log(N))^2 \epsilon^{-3} D K^{4m-1} \log(\delta^{-1}) \right)
  \end{align*}
\end{proof}

\begin{lemma}
  The runtime of \Cref{alg:fuzzy-coresets} is bounded by
  \[ \cO(NDK+ \abs{S}) \ . \]
\end{lemma}
\begin{proof}
  First, note that the $(\alpha,\beta)$-bicriteria algorithm from \cite{aggarwal09} takes time $\cO(NDK\log(\delta^{-1}))$.

  Second, we analyse the runtime of \Cref{alg:chen}.
  To determine $\OneTo{j}{F}_0$ such that $x\in L_{k,j}$ we only need to compute $\lceil \log(\norm{x-\mu_k}(R))\rceil$.
  Hence, computing all $X_{k,j}$ takes time $\cO(NDK)$.
  Sampling $\abs{S}$ points afterwards takes linear time.
\end{proof}

  \section{Applications}\label{sec:appl}

In the following, we present two applications of our coresets for fuzzy $K$-means.
In general, our coresets can be plugged in before any application of an algorithm that tries to solve fuzzy $K$-means and can handle weighted data sets.
If the applied algorithm's runtime does not depend on the actual weights, then this leads to a significant reduction in runtime.
We show that this yields a faster PTAS for fuzzy $K$-means than the ones presented before \cite{bbb16}.
Furthermore, we argue that our coresets can be maintained in an insertion-only streaming setting.

\subsection{Speeding up Aproximation}\label{subsec:approx}

We start by presenting an improved analysis of a simple sampling-based PTAS for the fuzzy $K$-means problem.
Our analysis exploits that the algorithm can ignore the weights of the data points and still obtain an approximation guarantee of $(1+\epsilon)$ for the weighted problem.
This means, that the algorithms runtime is independent of the weights, and thus can be significantly reduced by applying it to a coreset instead of the original data.
The first ingredient is the following, previosuly presented, soft-to-hard lemma.

\begin{lemma}[\cite{bbb16}]\label{lem:softtohard}
	Let $\epsilon\in(0,1)$, $r:X\times [K] \rightarrow [0,1]$ be a membership function, and let $M = \{\mu_1,\dots,\mu_K\}$ be the corresponding optimal mean vectors.

	If $\forall \OneTo{k}{K}: r(X,w,k) \geq 16K\wmax(X)/\epsilon$, then there exist pairwise disjoint sets $C_1,\ldots,C_K \subseteq X$ such that for all $\OneTo{k}{K}$
	\begin{align*}
		 w(C_k) & \geq \frac{r(X,w,k)}{2} \ , \\
		 \norm{\mu_w(C_k) - \mu_k}^2 & \leq \frac{\epsilon}{r(X,w,k)} \phi_k(X,w,M,r)  \mbox{, and}\\
		 \km(C_k) & \leq 4K \cdot \phi_k(X,w,M,r) \ .
	\end{align*}
\end{lemma}

We combine this with a classical concentration bound by Inaba et al.

\begin{lemma}[\cite{inaba94}]\label{lem:inaba}
	Let $P\subset\R^D$, $n\in\N$, $\delta\in(0,1)$, and let $S$ be a set of $n$ points drawn uniformly at random from $P$.
	Then we have
	\[ \Pr\left(\norm{\mu_\one(S)-\mu_\one(P)}^2 \leq \frac{1}{\delta n} \frac{\km(P,\one)}{\abs{P}}\right) \geq 1-\delta \ . \]
\end{lemma}

\begin{corollary}\label{cor:inaba}
	Let $X\subset\R^D$, $\wfuncN{w}{X}$, $K\in\N$, $\epsilon\in(0,1)$, and let $C_1,\dots,C_K\subseteq X$ be non-empty subsets of $X$.
	There exist $K$ multisets $S_1,\dots,S_K\subseteq X$, such that
	\[ \forall \OneTo{k}{K}: \abs{S_k} = \frac{2}{\epsilon} \mbox{ and } \norm{\mu_\one(S_k)-\mu_w(C_k)}^2 \leq \epsilon \frac{\km(C_k,w)}{w(C_k)} \ . \]
\end{corollary}

We can find means of subsets obtained from applying the soft-to-hard lemma to the clusters of an optimal fuzzy $K$-means solution by derandomizing Inaba's sampling technique.

\begin{algorithm}
	\SetEndCharOfAlgoLine{}
	\caption{\textsc{Derandomized Sampling}}\label{alg:ptas-sampling}
	\KwIn{$X\subset\R^D$, $K\in\N$, $\epsilon\in(0,1)$}
	$\cT \gets \{ \mu_\one(S) \mid S\subseteq X, \abs{S} = \frac{64K}{\epsilon} \}$\;
	\tcc{$S$ as multisets -- Points can occur multiple times in each $S$ and are counted with multiplicity.}
	$M \gets \arg\min_{T\subseteq\cT, \abs{T} = K} \{\phi(X,w,T)\}$\;
	\Return{$M$}\;
\end{algorithm}

\begin{theorem}\label{thm:ptas-no-coreset}
	\Cref{alg:ptas-sampling} computes $M\subset\R^D$ with $\abs{M} = K$, such that
	\[ \phi(X,w,M) \leq (1+\epsilon)\phi^{opt}(X,w) \]
	in time
	\[ DN^{\cO(K^2/\epsilon)} \ . \]
\end{theorem}
\begin{proof}
	We analyse the result $M$ of \Cref{alg:ptas-sampling}.
	Let $M^*$, $r^*$ be an optimal solution to the fuzzy $K$-means problem on $X$, $w$.
	Let $X_c$ be a modified point set, which contains $c$ copies of every point $x\in X$, where
	\[ c = \left\lceil \frac{\gamma K \wmax(X)}{\epsilon\min_{\OneTo{k}{K}}r^*(X,w,k)}\right\rceil  \ , \]
	for some large enough constant $\gamma$.
	For all sets of means $M$ and all membership functions $r$, we have $\phi(X_c,w,M,r) = c\cdot \phi(X,w,M,r)$.
	Thus, $M^*$ and $r^*$ (where $r^*(y,k) = r^*(x,k)$ for all $\OneTo{k}{K}$ and $x\in X, y\in X_c$ with $x=y$) are also optimal for the modified instance $X_c$.
	Observe, that for all $\OneTo{k}{K}$ we have
	\[ r^*(X_c,w,k) \geq \sum_{x\in X} \frac{\gamma K \wmax(X)}{\epsilon\min_{\OneTo{k}{K}}r^*(X,w,k)}w(x)r^*(x,k)^m \geq \frac{\gamma K \wmax(X)}{\epsilon} \geq \frac{64 K \wmax(X)}{\epsilon} \ . \]
	By applying \Cref{lem:softtohard} with respect to $X_c$, $w$, and $\epsilon/4$ we obtain that there exist disjoint sets $C_1,\dots,C_K\subseteq X_c$ such that for all $\OneTo{k}{K}$ we have
	\begin{align}
		w(C_k) &\geq \frac{r^*(X_c,w,k)}{2} \ , \label{eq:weight} \\
		\norm{\mu_w(C_k)-\mu^*_k}^2 &\leq \frac{\epsilon}{4r^*(X_c,w,k)}\phi_k(X_c,w,M^*,r^*)yellow \mbox{ , and} \label{eq:means} \\
		\km(C_k,w) &\leq 4K \cdot \phi_k(X_c,w,M^*,r^*) \ . \label{eq:cost}
	\end{align}
	Next, we apply \Cref{cor:inaba} to $X_c$, $w$, $K$, $\epsilon/(32K)$, and $C_1,\dots,C_K$.
	We obtain that there exist $S_1,\dots,S_K\subseteq X_c$ such that for all $\OneTo{k}{K}$ we have $\abs{S_k} = 64K/\epsilon$ and
	\begin{align}
		\norm{\mu_\one(S_k)-\mu_w(C_k)}^2 \leq \epsilon/(32K)\km(C_k,w)/w(C_k) \ . \label{eq:inabameans}
	\end{align}
	Since $X_c$ consists of copies of points from $X$, we can conclude that $S_1,\dots,S_K\subseteq X$, if we treat the $S_k$ as multisets, i.e. allow the same point to appear multiple times in the same set.
	Hence, by choice of $M$, as made by \Cref{alg:ptas-sampling}, we have $\phi(X,w,M) \leq \phi(X,w,\{\mu_\one(S_k)\}_{\OneTo{k}{K}})$.
	Plugging all this together we can bound the cost of $M$ as follows
	\begin{align*}
		\phi(X,w,M) &\leq \phi(X,w,\{\mu_\one(S_k)\}_{\OneTo{k}{K}}) = \frac{1}{c}\phi(X_c,w,\{\mu_\one(S_k)\}_{\OneTo{k}{K}}) \\
			&\leq \frac{1}{c}\phi(X_c,w,\{\mu_\one(S_k)\}_{\OneTo{k}{K}},r^*) = \frac{1}{c}\sum_{x\in X_c}\sum_{\OneTo{k}{K}} w(x)r^*(x,k)^m\norm{x-\mu_\one(S_k)}^2 \\
			&\leq \phi(X,w,r^*) + \frac{2}{c}\sum_{x\in X_c}\sum_{\OneTo{k}{K}} w(x) r^*(x,k)^m\norm{\mu^*_k - \mu_w(C_k)}^2\\
			&\phantom{\leq \phi(X,w,r^*)} + \frac{2}{c}\sum_{x\in X_c}\sum_{\OneTo{k}{K}} w(x) r^*(x,k)^m\norm{\mu_w(C_k) - \mu_\one(S_k)}^2 \tag{by $2$-approximate triangle inequality} \\
			&\leq \phi^{opt}(X,w) + \frac{\epsilon}{2c}\sum_{\OneTo{k}{K}} \phi_k(X_c,w,M^*,r^*) \tag{by \eqref{eq:means}} \\
			&\phantom{\leq \phi(X,w,r^*)}+ \frac{\epsilon}{c16K}\sum_{\OneTo{k}{K}} \frac{\km(C_k,w)}{w(C_k)} \sum_{x\in X_c}w(x)r^*(x,k)^m \tag{by \eqref{eq:inabameans}} \\
			&\leq (1+\epsilon/2)\phi^{opt}(X,w) + \frac{\epsilon}{2c}\sum_{\OneTo{k}{K}} \phi_k(X_c,w,M^*,r^*) \tag{by \eqref{eq:weight} and \eqref{eq:cost}} \\
			&= (1+\epsilon) \phi^{opt}(X,w) \ .
	\end{align*}
	Bounding the runtime of \Cref{alg:ptas-sampling} is straightforward.
	We have to evaluate the cost of $\abs{\cT}^K$ different fuzzy $K$-means solution, each evaluation costing $\cO(NDK)$.
	Hence, the total runtime is bounded by
	\[ \cO(NDK\abs{\cT}^K) = \cO(NDK(N^{64K/\epsilon})^K) = DN^{\cO(K^2/\epsilon)} \ . \]
\end{proof}

Recall, that the runtime of \Cref{alg:ptas-sampling} is independent of point weights.
Hence, we obtain a more efficient algorithm by first computing a coreset using \Cref{thm:coreset} and then applying \Cref{alg:ptas-sampling} to this coreset instead of the original data set.
In the following, we formally only state an unweighted version of our result.

\begin{corollary}\label{cor:ptas}
	There exists an algorithm which, given $X\subset\R^D$, $K\in\N$, and $\epsilon\in(0,1)$, computes a set $M\subset\R^D$ with $\abs{M} = K$, such that with constant probability
	\[ \phi(X,\one,M) \leq (1+\epsilon)\phi^{opt}(X,\one) \]
	in time
	\[ \cO(NDK) + \left( \log(N) D \right)^{\cO(K^2/\epsilon\log(K/\epsilon))} \ . \]
\end{corollary}
\begin{proof}
	Given $X$, $K$, and $\epsilon$, apply \Cref{thm:coreset} (with $\epsilon/3$) to obtain, with constant probability, an $\epsilon/3$-coreset $(S,w_S)$ of $(X,\one)$ and let $M$ be the output of \Cref{alg:ptas-sampling} given $S$, $w_S$, and $\epsilon/3$.
	We obtain
	\begin{align*}
		\phi(S,w_S,M) &\leq (1+\epsilon/3)\phi^{opt}(S,w_S) \leq (1+\epsilon/3)\phi(S,w_S,M_X^*) \\
			&\leq (1+\epsilon/3)^2 \phi^{opt}(X,\one) \leq (1+\epsilon) \phi^{opt}(X,\one) \ ,
	\end{align*}
	where $M_X^*$ is an optimal set of means with respect to $X$.
	The overall runtime is
	\begin{align*}
		\cO(NDK) + D(\abs{S})^{\cO(K^2/\epsilon)} &= \cO(NDK) + (\log(N) D K / \epsilon)^{\cO(K^2/\epsilon)} \\
			&= \cO(NDK) + (\log(N) D)^{\cO(K^2/\epsilon\log(K/\epsilon))} \ .
	\end{align*}
\end{proof}

The algorithm from \Cref{cor:ptas} can also be applied to weighted data sets.
However, its runtime is not independent of these weights.
We argued that the runtime of the PTAS from \Cref{thm:ptas-no-coreset} is independent of any weights, but this is not true for the coreset construction.
Hence, weight functions have the same impact on the runtime, which we discussed in \Cref{sec:coresets} in regard to the coreset construction.

Nonetheless, our algorithm has significant advantages over previously presented $(1+\epsilon)$-approximation algorithms for fuzzy $K$-means.
The runtimes of all algorithms presented in \cite{bbb16} have an exponential dependency on the dimension $D$ or contain a term $N^{\cO(\mathrm{poly}(K,1/\epsilon))}$.
Our result constitutes the first algorithm with a non-exponential dependence on $D$ whose only exponential term is of the form $\log(N)^{\cO(\mathrm{poly}(K,1/\epsilon))}$.

Strictly speaking, applying \Cref{alg:ptas-sampling} directly on $X$ is faster if $D\in\Omega(N)$.
However, in that case we can apply the lemma of Johnson and Lindenstrauss \cite{jl84} to replace $D$ by $\log(N)/\epsilon^2$

\subsection{Streaming Model}\label{subsec:streaming}

We give a brief overview of the method to maintain coresets in a streaming model presented in \cite{Feldman13}.
It is an improved version of the techniques previously used by \cite{Chen09} and \cite{harpeled03}.
The central observation is that the union of coresets of two input data sets is a coreset of the union of the data sets.
Whenever a sufficient (depending on the coreset construction) number of points has arrived in the stream, we compute a coreset of these points.
After two coresets have been computed, we merge them into a larger coreset of all points that have arrived, so far.
Following two of these merge operations, we merge the two larger coresets into one even larger one.
This continues in the fashion of a binary tree.
Since our coresets for fuzzy $K$-means fulfil all requirements to apply this approach, it can also be used to maintain fuzzy $K$-means coresets in the streaming model.

\begin{theorem}
	Given $N$ data points in a stream (one-by-one) and $\epsilon\in(0,1)$ one can maintain, with high probability, an $\epsilon$-coreset for the fuzzy $K$-means problem, of the points seen so far, using $\cO(DK^{4m-1}\cdot \mathrm{polylog} (N/\epsilon))$ memory.
	Arriving data points cause an update with an amortized runtime of $\cO(DK \cdot \mathrm{polylog}(NDK/\epsilon))$.
\end{theorem}

  \section{Discussion and Outlook}

We proved that a parameter tuned version of Chen's construction yields the first coresets for the fuzzy $K$-means problem.
While there are a plethora of coreset constructions for $K$-means, Chen's construction is the best purely sampling based approach.
More efficient techniques, for example $\epsilon$-nets \cite{harpeled05} or subspace approaches like low-rank approximation \cite{Feldman13}, heavily rely on the partitioning of the input set that a $K$-means solution induces.
So far, we have not found a way to apply these to the, already notoriously hard to analyse, fuzzy $K$-means objective function.
This is because the membership function essentially introduces an unknown weighting on the points.
Hence, when the data set is partitioned or projected into some subspace without respecting this weighting, we introduce a factor $K^{\cO(1)}$ to the cost estimation.
It has proven difficult to control these additional factors.
Partly for these reasons, there is still a large number of open questions regarding fuzzy $K$-means.

In this paper, we almost match the asymptotic runtime of the fastest $(1+\epsilon)$-approximation algorithms for $K$-means.
However, even assuming constant $K$, our algorithms lack practicality due to the large constants hidden in the $\cO$.
Hence, this raises interesting follow-up questions.
Is there an efficient approximation algorithm for fuzzy $K$-means with a constant approximation factor?
What can be done in terms of bicriteria algorithms, i.e. if we are allowed to chose more than $K$ means?
In regard to the complexity of fuzzy $K$-means it is interesting to examine whether one can show that there is no true PTAS (polynomial runtime in $N$, $D$, and $K$) for fuzzy $K$-means, as it was shown for $K$-means \cite{awasthi15}.
Finally, can we relate the hardness of fuzzy $K$-means directly to $K$-means?

  \bibliographystyle{apalike}
  \bibliography{references}

\end{document}